\documentclass{article}

\usepackage{microtype}
\usepackage{graphicx}
\usepackage{subfigure}
\usepackage{booktabs} 

\graphicspath{{figs/}}

\usepackage{amsmath,amsthm,amssymb}

\newcommand{\real}{\mathbb{R}}
\DeclareMathOperator{\mean}{E}

\DeclareMathOperator{\tr}{tr}
\DeclareMathOperator*{\argmin}{argmin}
\DeclareMathOperator{\mlp}{MLP}
\DeclareMathOperator{\relu}{ReLU}
\DeclareMathOperator{\softmax}{softmax}

\theoremstyle{definition} 
\theoremstyle{plain}      
\theoremstyle{plain}      \newtheorem{theorem}{Theorem}
\theoremstyle{plain}      
\theoremstyle{plain}      

\usepackage{hyperref}

\usepackage[accepted]{icml2019}

\icmltitlerunning{DAG-GNN: DAG Structure Learning with Graph Neural Networks}

\begin{document}

\twocolumn[
\icmltitle{DAG-GNN: DAG Structure Learning with Graph Neural Networks}



\icmlsetsymbol{equal}{*}

\begin{icmlauthorlist}
\icmlauthor{Yue Yu}{equal,lehigh}
\icmlauthor{Jie Chen}{equal,mitibm,ibm}
\icmlauthor{Tian Gao}{ibm}
\icmlauthor{Mo Yu}{ibm}
\end{icmlauthorlist}

\icmlaffiliation{lehigh}{Lehigh University}
\icmlaffiliation{mitibm}{MIT-IBM Watson AI Lab}
\icmlaffiliation{ibm}{IBM Research}

\icmlcorrespondingauthor{Yue Yu}{yuy214@lehigh.edu}
\icmlcorrespondingauthor{Jie Chen}{chenjie@us.ibm.com}

\icmlkeywords{DAG, structure learning, graph neural network, variational autoencoder}

\vskip 0.3in
]



\printAffiliationsAndNotice{\icmlEqualContribution} 

\begin{abstract}
  Learning a faithful directed acyclic graph (DAG) from samples of a joint distribution is a challenging combinatorial problem, owing to the intractable search space superexponential in the number of graph nodes. A recent breakthrough formulates the problem as a continuous optimization with a structural constraint that ensures acyclicity (Zheng et al., 2018). The authors apply the approach to the linear structural equation model (SEM) and the least-squares loss function that are statistically well justified but nevertheless limited. Motivated by the widespread success of deep learning that is capable of capturing complex nonlinear mappings, in this work we propose a deep generative model and apply a variant of the structural constraint to learn the DAG. At the heart of the generative model is a variational autoencoder parameterized by a novel graph neural network architecture, which we coin DAG-GNN. In addition to the richer capacity, an advantage of the proposed model is that it naturally handles discrete variables as well as vector-valued ones. We demonstrate that on synthetic data sets, the proposed method learns more accurate graphs for nonlinearly generated samples; and on benchmark data sets with discrete variables, the learned graphs are reasonably close to the global optima. The code is available at \url{https://github.com/fishmoon1234/DAG-GNN}.
\end{abstract}

\section{Introduction}
Bayesian Networks (BN) have been widely used in machine learning applications~\citep{glymour1999computation,ott2004finding}. The structure of a BN takes the form of a directed acyclic graph (DAG) and plays a vital part in causal inference~\citep{pearl88} with many applications in medicine, genetics, economics, and epidemics. Its structure learning problem is however NP-hard~\citep{chickering2004large} and stimulates a proliferation of literature.

Score-based methods generally formulate the structure learning problem as optimizing a certain score function with respect to the unknown (weighted) adjacency matrix $A$ and the observed data samples, with a combinatorial constraint stating that the graph must be acyclic. The intractable search space (with a complexity superexponential in the number of graph nodes) poses substantial challenges for optimization. Hence, for practical problems in a scale beyond small, approximate search often needs to be employed with additional structure assumption~\citep{nie2014advances,chow1968approximating,scanagatta2015learning,chen2016learning}.

Recently, \citet{Zheng2018} formulate an equivalent acyclicity constraint by using a continuous function of the adjacency matrix (specifically, the matrix exponential of $A\circ A$). This approach drastically changes the combinatorial nature of the problem to a continuous optimization, which may be efficiently solved by using maturely developed blackbox solvers. The optimization problem is nevertheless nonlinear, thus these solvers generally return only a stationary-point solution rather than the global optimum. Nevertheless, the authors show that empirically such local solutions are highly comparable to the global ones obtained through expensive combinatorial search.

With the inspiring reformulation of the constraint, we revisit the objective function. The score-based objective functions generally make assumptions of the variables and the model class. For example, \citet{Zheng2018} demonstrate on the linear structural equation model (SEM) with a least-squares loss. While convenient, such assumptions are often restricted and they may not correctly reflect the actual distribution of real-life data.

Hence, motivated by the remarkable success of deep neural networks, which are arguably universal approximators, in this work we develop a graph-based deep generative model aiming at better capturing the sampling distribution faithful to the DAG. To this end, we employ the machinery of variational inference and parameterize a pair of encoder/decoder with specially designed graph neural networks (GNN). The objective function (the score), then, is the evidence lower bound. Different from the current flourishing designs of GNNs~\citep{Bruna2014,Defferrard2016,Li2016,Kipf2017,Hamilton2017,Gilmer2017,Chen2018,Velickovic2018}, the proposed ones are generalized from linear SEM, so that the new model performs at least as well as linear SEM when the data is linear.

Our proposal has the following distinct features and advantages. First, the work is built on the widespread use of deep generative models (specifically, variational autoencoders, VAE~\citep{Kingma2014}) that are able to capture complex distributions of data and to sample from them. Under the graph setting, the weighted adjacency matrix is an explicit parameter, rather than a latent structure, learnable together with other neural network parameters. The proposed network architecture has not been used before.

Second, the framework of VAE naturally handles various data types, notably not only continuous but also discrete ones. All one needs to do is to model the likelihood distribution (decoder output) consistent with the nature of the variables.

Third, owing to the use of graph neural networks for parameterization, each variable (node) can be not only scalar-valued but also vector-valued. These variables are considered node features input to/output of the GNNs.

Fourth, we propose a variant of the acyclicity constraint more suitable for implementation under current deep learning platforms. The matrix exponential suggested by~\citet{Zheng2018}, while mathematically elegant, may not be implemented or supported with automatic differentiation in all popular platforms. We propose a polynomial alternative more practically convenient and as numerically stable as the exponential.

We demonstrate the effectiveness of the proposed method on synthetic data generated from linear and nonlinear SEMs, benchmark data sets with discrete variables, and data sets from applications. For synthetic data, the proposed DAG-GNN outperforms DAG-NOTEARS, the algorithm proposed by~\citet{Zheng2018} based on linear SEM. For benchmark data, our learned graphs compare favorably with those obtained through optimizing the Bayesian information criterion by using combinatorial search.

\section{Background and Related Work}
A DAG $G$ and a joint distribution $\mathcal{P}$ are \emph{faithful} to each other if all and only the conditional independencies true in $\mathcal{P}$ are entailed by $G$~\citep{pearl88}. The faithfulness condition enables one to recover  $G$ from $\mathcal{P}$. Given independent and iid samples $D$ from an unknown distribution corresponding to a faithful but unknown DAG, \emph{structure learning} refers to recovering the DAG from $D$.

Many exact and approximate algorithms for learning DAG from data have been developed, including score-based and constraint-based approaches~\citep{PCalgorithm,chickering02,koivisto2004exact,silander06,jaakkola2010,cussens11,Yuan13learning}. Score-based methods generally use a score to measure the goodness of fit of different graphs over data; and then use a search procedure---such as hill-climbing~\citep{heckerman1995learning,MMPCcor,HCimap}, forward-backward search~\citep{chickering02}, dynamic programming~\citep{smDP,silander06}, A$^*$~\citep{Yuan13learning}, or integer programming~\citep{jaakkola2010,cussens11,cussens2016polyhedral}---in order to find the best graph. Commonly used Bayesian score criteria, such as BDeu and Bayesian information criterion (BIC), are decomposable, consistent, locally consistent~\citep{chickering02}, and score equivalent~\citep{heckerman1995learning}.

To make the DAG search space tractable, approximate methods make additional assumptions such as bounded tree-width~\citep{nie2014advances}, tree-like structures~\citep{chow1968approximating}, approximation~\citep{scanagatta2015learning}, and other constraints about the DAG~\citep{chen2016learning}. Many bootstrap~\citep{friedman1999data} and sampling-based structure learning algorithms~\citep{madigan1995bayesian,friedman2003being,eaton2012bayesian,grzegorczyk2008improving,niinimaki2013annealed,niinimaki2012partial,he2016structure} are also proposed to tackle the expensive search problem. 

Constraint-based methods, in contrast, use (conditional) independence tests to test the existence of  edges between each pair of variables.  Popular algorithms include SGS~\citep{spirtes2000causation},  PC~\citep{spirtes2000causation}, IC~\citep{pearl2003causality}, and FCI~\citep{spirtes1995causal,zhang2008completeness}. Recently, there appears a suite of hybrid algorithms that combine score-based and constraint-based methods, such as MMHC~\citep{Tsamardinos03}, and apply constraint-based methods to multiple environments~\citep{mooij2016joint}.

Due to the NP-hardness, traditional DAG learning methods usually deal with discrete variables, as discussed above, or jointly Gaussian variables~\citep{mohan2012structured,mohammadi2015bayesian}.
Recently, a new continuous optimization approach is proposed~\citep{Zheng2018}, which transforms the discrete search procedure into an equality constraint. This approach enables a suite of continuous optimization techniques such as gradient descent to be used. The approach achieves good structure recovery results, although it is applied to only linear SEM for ease of exposition.

Neural-network approaches started to surface only very recently. \citet{Kalainathan2018} propose a GAN-style (generative adversarial network) method, whereby a separate generative model is applied to each variable and a discriminator is used to distinguish between the joint distributions of real and generated samples. The approach appears to scale well but acyclicity is not enforced.

\section{Neural DAG Structure Learning}
Our method learns the weighted adjacency matrix of a DAG by using a deep generative model that generalizes linear SEM, with which we start the journey.

\subsection{Linear Structural Equation Model}
Let $A\in\real^{m\times m}$ be the weighted adjacency matrix of the DAG with $m$ nodes and $X\in\real^{m\times d}$ be a sample of a joint distribution of $m$ variables, where each row corresponds to one variable. In the literature, a variable is typically a scalar, but it can be trivially generalized to a $d$-dimensional vector under the current setting. The linear SEM model reads
\begin{equation}\label{eqn:linear.sem}
  X = A^TX + Z,
\end{equation}
where $Z\in\real^{m\times d}$ is the noise matrix. When the graph nodes are sorted in the topological order, the matrix $A$ is strictly upper triangular. Hence, ancestral sampling from the DAG is equivalent to generating a random noise $Z$ followed by a triangular solve
\begin{equation}\label{eqn:linear.sem.sampling}
  X=(I-A^T)^{-1}Z.
\end{equation}

\subsection{Proposed Graph Neural Network Model}
Equation~\eqref{eqn:linear.sem.sampling} may be written as $X=f_A(Z)$, a general form recognized by the deep learning community as an abstraction of parameterized graph neural networks that take node features $Z$ as input and return $X$ as high level representations. Nearly all graph neural networks~\citep{Bruna2014,Defferrard2016,Li2016,Kipf2017,Hamilton2017,Gilmer2017,Chen2018,Velickovic2018} can be written in this form. For example, the popular GCN~\citep{Kipf2017} architecture reads
\[
X = \widehat{A}\cdot\relu(\widehat{A}ZW^1)\cdot W^2,
\]
where $\widehat{A}$ is a normalization of $A$ and $W^1$ and $W^2$ are parameter matrices.

Owing to the special structure~\eqref{eqn:linear.sem.sampling}, we propose a new graph neural network architecture
\begin{equation}\label{eqn:decoder}
X=f_2((I-A^T)^{-1}f_1(Z)).
\end{equation}
The parameterized functions $f_1$ and $f_2$ effectively perform (possibly nonlinear) transforms on $Z$ and $X$, respectively. If $f_2$ is invertible, then~\eqref{eqn:decoder} is equivalent to $f_2^{-1}(X) = A^Tf_2^{-1}(X) + f_1(Z)$, a generalized version of the linear SEM~\eqref{eqn:linear.sem}.

We will defer the instantiation of these functions in a later subsection. One of the reasons is that the activation in $f_2$ must match the type of the variable $X$, a subject to be discussed together with discrete variables.

\subsection{Model Learning with Variational Autoencoder}
Given a specification of the distribution of $Z$ and samples $X^1,\ldots,X^n$, one may learn the generative model~\eqref{eqn:decoder} through maximizing the log-evidence
\[
\frac{1}{n}\sum_{k=1}^n\log p(X^k)=\frac{1}{n}\sum_{k=1}^n\log\int p(X^k|Z)p(Z)\,dZ,
\]
which, unfortunately, is generally intractable. Hence, we appeal to variational Bayes.

To this end, we use a variational posterior $q(Z|X)$ to approximate the actual posterior $p(Z|X)$. The net result is the evidence lower bound (ELBO)
\[
L_{\text{ELBO}}=\frac{1}{n}\sum_{k=1}^nL_{\text{ELBO}}^k,
\]
with
\begin{multline}\label{eqn:elbo.i}
L_{\text{ELBO}}^k \equiv -D_{\text{KL}}\Big(q(Z|X^k) \,||\, p(Z)\Big) \\
+\mean_{q(Z|X^k)}\Big[\log p(X^k|Z)\Big].
\end{multline}
Each individual term $L_{\text{ELBO}}^k$ departs from the log-evidence by $D_{\text{KL}}\Big(q(Z|X^k) \,||\, p(Z|X^k)\Big)\ge0$, the KL-divergence between the variational posterior and the actual one.

The ELBO lends itself to a variational autoencoder (VAE)~\citep{Kingma2014}, where given a sample $X^k$, the encoder (inference model) encodes it into a latent variable $Z$ with density $q(Z|X^k)$; and the decoder (generative model) tries to reconstruct $X^k$ from $Z$ with density $p(X^k|Z)$. Both densities may be parameterized by using neural networks.

Modulo the probability specification to be completed later, the generative model~\eqref{eqn:decoder} discussed in the preceding subsection plays the role of the decoder. Then, we propose the corresponding encoder
\begin{equation}\label{eqn:encoder}
  Z = f_4((I-A^T)f_3(X)),
\end{equation}
where$f_3$ and $f_4$ are parameterized functions that conceptually play the inverse role of $f_2$ and $f_1$, respectively.

\begin{figure*}[ht]
  \centering
  \includegraphics[width=.98\linewidth]{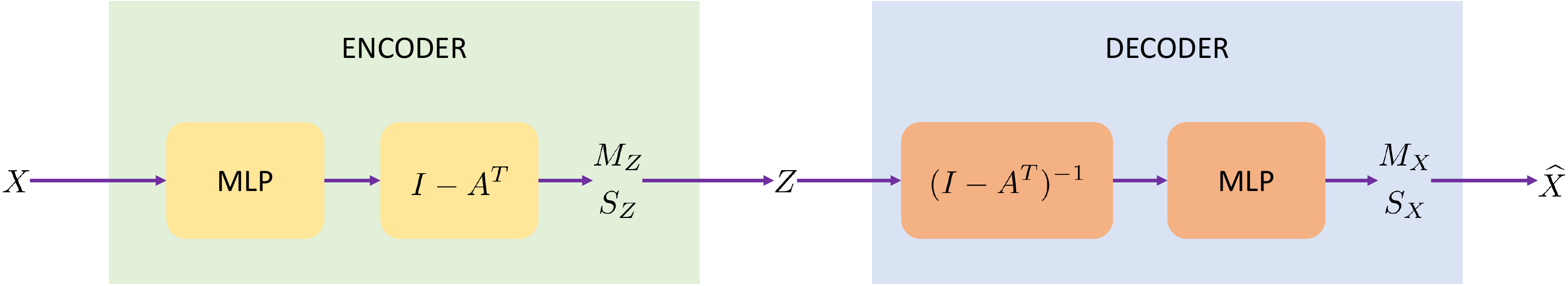}
  \caption{Architecture (for continuous variables). In the case of discrete variables, the decoder output is changed from $M_X,S_X$ to $P_X$.}
  \label{fig:architecture}
\end{figure*}

\subsection{Architecture and Loss Function}
To complete the VAE, one must specify the distributions in~\eqref{eqn:elbo.i}. Recall that for now both $X^k$ and $Z$ are $m\times d$ matrices. For simplicity, the prior is typically modeled as the standard matrix normal $p(Z) = \mathcal{MN}_{m\times d}(0,I,I)$.

For the inference model, we let $f_3$ be a multilayer perceptron (MLP) and $f_4$ be the identity mapping. Then, the variational posterior $q(Z|X)$ is a factored Gaussian with mean $M_Z\in\real^{m\times d}$ and standard deviation $S_Z\in\real^{m\times d}$, computed from the encoder
\begin{equation}\label{eqn:encoder.p}
[M_Z | \log S_Z] = (I-A^T)\mlp(X, W^1, W^2),
\end{equation}
where
$\mlp(X, W^1, W^2) := \relu(XW^1)W^2$,
and $W^1$ and $W^2$ are parameter matrices.

For the generative model, we let $f_1$ be the identity mapping and $f_2$ be an MLP. Then, the likelihood $p(X|Z)$ is a factored Gaussian with mean $M_X\in\real^{m\times d}$ and standard deviation $S_X\in\real^{m\times d}$, computed from the decoder
\begin{equation}\label{eqn:decoder.p}
[M_X | \log S_X] = \mlp((I-A^T)^{-1}Z, W^3, W^4),
\end{equation}
where $W^3$ and $W^4$ are parameter matrices.

One may switch the MLP and the identity mapping inside each of the encoder/decoder, but we find that the performance is less competitive. One possible reason is that the current design~\eqref{eqn:decoder.p} places an emphasis on the nonlinear transform of a sample $(I-A^T)^{-1}Z$ from linear SEM, which better captures nonlinearity.

Based on~\eqref{eqn:encoder.p} and~\eqref{eqn:decoder.p}, the KL-divergence term in the ELBO~\eqref{eqn:elbo.i} admits a closed form
\begin{multline}\label{eqn:loss.kl}
D_{\text{KL}}\Big(q(Z|X) \,||\, p(Z)\Big) = \\
\frac{1}{2}\sum_{i=1}^m\sum_{j=1}^d (S_Z)_{ij}^2 + (M_Z)_{ij}^2 - 2\log(S_Z)_{ij} - 1,
\end{multline}
and the reconstruction accuracy term may be computed with Monte Carlo approximation
\begin{multline}\label{eqn:loss.recons}
\mean_{q(Z|X)}\Big[\log p(X|Z)\Big] \approx \\
\frac{1}{L}\sum_{l=1}^L\sum_{i=1}^m\sum_{j=1}^d -\frac{(X_{ij}-(M_X^{(l)})_{ij})^2}{2(S_X^{(l)})_{ij}^2} - \log (S_X^{(l)})_{ij} - c,
\end{multline}
where $c$ is a constant and $M_X^{(l)}$ and $S_X^{(l)}$ are the outputs of the decoder~\eqref{eqn:decoder.p} by taking as input Monte Carlo samples $Z^{(l)}\sim q(Z|X)$, for $l=1,\ldots,L$.

Note that under the autoencoder framework, $Z$ is considered latent (rather than the noise in linear SEM). Hence, the column dimension of $Z$ may be different from $d$. From the neural network point of view, changing the column dimension of $Z$ affects only the sizes of the parameter matrices $W^2$ and $W^3$. Sometimes, one may want to use a smaller number than $d$ if he/she observes that the data has a smaller intrinsic dimension.

An illustration of the architecture is shown in Figure~\ref{fig:architecture}.

\subsection{Discrete Variables}
One advantage of the proposed method is that it naturally handles discrete variables. We assume that each variable has a finite support of cardinality $d$. 

Hence, we let each row of $X$ be a one-hot vector, where the ``on'' location indicates the value of the corresponding variable. We still use standard matrix normal to model the prior and factored Gaussian to model the variational posterior, with~\eqref{eqn:encoder.p} being the encoder. On the other hand, we need to slightly modify the likelihood to cope with the discrete nature of the variables.

Specifically, we let $p(X|Z)$ be a factored categorical distribution with probability matrix $P_X$, where each row is a probability vector for the corresponding categorical variable. To achieve so, we change $f_2$ from the identity mapping to a row-wise softmax and modify the decoder~\eqref{eqn:decoder.p} to
\begin{equation}\label{eqn:decoder.p.discrete}
P_X = \softmax(\mlp((I-A^T)^{-1}Z, W^3, W^4)).
\end{equation}
Correspondingly for the ELBO, the KL term~\eqref{eqn:loss.kl} remains the same, but the reconstruction term~\eqref{eqn:loss.recons} needs be modified to
\begin{equation}\label{eqn:loss.recons.discrete}
\mean_{q(Z|X)}\Big[\log p(X|Z)\Big] \approx \\
\frac{1}{L}\sum_{l=1}^L\sum_{i=1}^m\sum_{j=1}^d X_{ij}\log (P_X^{(l)})_{ij},
\end{equation}
where $P_X^{(l)}$ is the output of the decoder~\eqref{eqn:decoder.p.discrete} by taking as input Monte Carlo samples $Z^{(l)}\sim q(Z|X)$, for $l=1,\ldots,L$.

\subsection{Connection to Linear SEM}
One has seen from the forgoing discussions how the proposed model is developed from linear SEM: We apply nonlinearality to the sampling procedure~\eqref{eqn:linear.sem.sampling} of SEM, treat the resulting generative model as a decoder, and pair with it a variational encoder for tractable learning. Compared with a plain autoencoder, the variational version allows a modeling of the latent space, from which samples are generated.

We now proceed, in a reverse thought flow, to establish the connection between the loss function of the linear SEM considered in~\citet{Zheng2018} and that of ours. We first strip off the variational component of the autoencoder. This plain version uses~\eqref{eqn:encoder} as the encoder and~\eqref{eqn:decoder} as the decoder. For notational clarity, let us write $\widehat{X}$ as the output of the decoder, to distinguish it from the encoder input $X$. A typical sample loss to minimize is
\[
\frac{1}{2}\sum_{i=1}^m\sum_{j=1}^d(X_{ij}-\widehat{X}_{ij})^2
+ \frac{1}{2}\sum_{i=1}^m\sum_{j=1}^d Z_{ij}^2,
\]
where the first term is the reconstruction error and the second term is a regularization of the latent space. One recognizes that the reconstruction error is the same as the negative reconstruction accuracy~\eqref{eqn:loss.recons} in the ELBO, up to a constant, if the standard deviation $S_X$ is $1$, the mean $M_X$ is taken as $\widehat{X}$, and only one Monte Carlo sample is drawn from the variational posterior. Moreover, the regularization term is the same as the KL-divergence~\eqref{eqn:loss.kl} in the ELBO if the standard deviation $S_Z$ is $1$ and the mean $M_Z$ is taken as $Z$.

If we further strip off the (possibly nonlinear) mappings $f_1$ to $f_4$, then the encoder~\eqref{eqn:encoder} and decoder~\eqref{eqn:decoder} read, respectively, $Z=(I-A^T)X$ and $\widehat{X}=(I-A^T)^{-1}Z$. This pair results in perfect reconstruction, and hence the sample loss reduces to
\begin{equation}\label{eqn:loss.ls}
\frac{1}{2}\sum_{i=1}^m\sum_{j=1}^d Z_{ij}^2
= \frac{1}{2}\|(I-A^T)X\|_F^2,
\end{equation}
which is the least-squares loss used and justified by~\citet{Zheng2018}.

\subsection{Acyclicity Constraint}
Neither maximizing the ELBO~\eqref{eqn:elbo.i} nor minimizing the least-squares loss~\eqref{eqn:loss.ls} guarantees that the corresponding graph of the resulting $A$ is acyclic. \citet{Zheng2018} pair the loss function with an equality constraint, whose satisfaction ensures acyclicity.

The idea is based on the fact that the positivity of the $(i,j)$ element of the $k$-th power of a nonnegative adjacency matrix $B$ indicates the existence of a length-$k$ path between nodes $i$ and $j$. Hence, the positivity of the diagonal of $B^k$ reveals cycles. The authors leverage the trick that the matrix exponential admits a Taylor series (because it is analytic on the complex plane), which is nothing but a weighted sum of all nonnegative integer powers of the matrix. The coefficient of the zeorth power (the identity matrix $I_{m\times m}$) is $1$, and hence the trace of the exponential of $B$ must be exactly $m$ for a DAG. To satisfy nonnegativity, one may let $B$ be the elementwise square of $A$; that is, $B=A\circ A$.

Whereas the formulation of this acyclicity constraint is mathematically elegant, support of the matrix exponential may not be available in all deep learning platforms. To ease the coding effort, we propose an alternative constraint that is practically convenient.

\begin{theorem}\label{thm:tr}
  Let $A\in\real^{m\times m}$ be the (possibly negatively) weighted adjacency matrix of a directed graph. For any $\alpha>0$, the graph is acyclic if and only if
  \begin{equation}\label{eqn:dag.constraint}
    \tr[(I + \alpha A\circ A)^m] - m = 0.
  \end{equation}
\end{theorem}

We use~\eqref{eqn:dag.constraint} as the equality constraint when maximizing the ELBO. The computations of both $(I+\alpha B)^m$ and $\exp(B)$ may meet numerical difficulty when the eigenvalues of $B$ have a large magnitude. However, the former is less severe than the latter with a judicious choice of $\alpha$.

\begin{theorem}\label{thm:exp}
  Let $\alpha=c/m>0$ for some $c$ . Then for any complex $\lambda$, we have
  $(1+\alpha|\lambda|)^m \le e^{c|\lambda|}$.
\end{theorem}

In practice, $\alpha$ may be treated as a hyperparameter and its setting depends on an estimation of the largest eigenvalue of $B$ in magnitude. This value is the spectral radius of $B$, and because of nonnegativity, it is bounded by the maximum row sum according to the Perron--Frobenius theorem.

\subsection{Training}
Based on the foregoing, the learning problem is
\begin{align*}
\min_{A,\theta} \quad& f(A,\theta) \equiv -L_{\text{ELBO}} \\
\text{s.t.} \quad& h(A) \equiv \tr[(I + \alpha A\circ A)^m] - m = 0,
\end{align*}
where the unknowns include the matrix $A$ and all the parameters $\theta$ of the VAE (currently we have $\theta=\{W^1,W^2,W^3,W^4\}$). Nonlinear equality-constrained problems are well studied and we use the augmented Lagrangian approach to solve it. For completeness, we summarize the algorithm here; the reader is referred to standard textbooks such as Section 4.2 of~\citet{Bertsekas1999} for details and convergence analysis.

Define the augmented Lagrangian
\[
L_c(A,\theta,\lambda)=f(A,\theta)+\lambda h(A)+\frac{c}{2}|h(A)|^2,
\]
where $\lambda$ is the Lagrange multiplier and $c$ is the penalty parameter. When $c=+\infty$, the minimizer of $L_c(A,\theta,\lambda)$ must satisfy $h(A)=0$, in which case $L_c(A,\theta,\lambda)$ is equal to the objective function $f(A,\theta)$. Hence, the strategy is to progressively increase $c$, for each of which minimize the unconstrained augmented Lagrangian. The Lagrange multiplier $\lambda$ is correspondingly updated so that it converges to the one under the optimality condition.

There exist a few variants for updating $\lambda$ and increasing $c$, but a typical effective rule reads:
\begin{align}
  (A^k,\theta^k) &= \argmin_{A,\theta} L_{c^k}(A,\theta,\lambda^k), \label{eqn:sub}\\
  \lambda^{k+1} &= \lambda^k+c^kh(A^k),\\
  c^{k+1} &=
  \begin{cases}
    \eta c^k, & \text{if } |h(A^k)|>\gamma|h(A^{k-1})|, \\
    c^k, & \text{otherwise},
  \end{cases}
\end{align}
where $\eta>1$ and $\gamma<1$ are tuning parameters. We find that often $\eta=10$ and $\gamma=1/4$ work well.

The subproblem~\eqref{eqn:sub} may be solved by using blackbox stochastic optimization solvers, by noting that the ELBO is defined on a set of samples.

\section{Experiments}
In this section, we present a comprehensive set of experiments to demonstrate the effectiveness of the proposed method DAG-GNN. In Section~\ref{sec:data1}, we compare with DAG-NOTEARS, the method proposed by~\citet{Zheng2018} based on linear SEM, on synthetic data sets generated by sampling generalized linear models, with an emphasis on nonlinear data and vector-valued data ($d>1$). In Section~\ref{sec:data2}, we showcase the capability of our model with discrete data, often seen in benchmark data sets with ground truths for assessing quality. To further illustrate the usefulness of the proposed method, in Section~\ref{sec:data3} we apply DAG-GNN on a protein data set for the discovery of consensus protein signaling network, as well as a knowledge base data set for learning causal relations.

Our implementation is based on PyTorch~\citep{paszke2017automatic}. We use Adam~\citep{Kingma2015} to solve the subproblems~\eqref{eqn:sub}. To avoid overparameterization, we parameterize the variational posterior $q(Z|X)$ as a factored Gaussian with constant unit variance, and similarly for the likelihood $p(X|Z)$. When extracting the DAG, we use a thresholding value $0.3$, following the recommendation of~\citet{Zheng2018}. For benchmark and application data sets, we include a Huber-norm regularization of $A$ in the objective function to encourage more rapid convergence.





\subsection{Synthetic Data Sets}\label{sec:data1}
The synthetic data sets are generated in the following manner. We first generate a random DAG by using the Erd\H{o}s--R\'{e}nyi model with expected node degree 3, then assign uniformly random weights for the edges to obtain the weighted adjacency matrix $A$. A sample $X$ is generated by sampling the (generalized) linear model $X=g(A^TX)+Z$ with some function $g$ elaborated soon. The noise $Z$ follows standard matrix normal. When the dimension $d=1$, we use lowercase letters to denote vectors; that is, $x=g(A^Tx)+z$. We compare DAG-GNN with DAG-NOTEARS and report the structural Hamming distance (SHD) and false discovery rate (FDR), each averaged over five random repetitions. With sample size $n=5000$, we run experiments on four graph sizes $m\in\{10,20,50,100\}$. In Sections~\ref{sec:lineardata} and~\ref{sec:nonlineardata} we consider scalar-valued variables ($d=1)$ and in Section~\ref{sec:multiddata} vector-valued variables ($d>1$).

\subsubsection{Linear Case}\label{sec:lineardata}
This case is the linear SEM model, with $g$ being the identity mapping. The SHD and FDR are plotted in Figure~\ref{fig:linear}. One sees that the graphs learned by the proposed method are substantially more accurate than those by DAG-NOTEARS when the graphs are large.

\begin{figure}[!ht]
\centering
\subfigure{\scalebox{0.38}[0.35]{\includegraphics{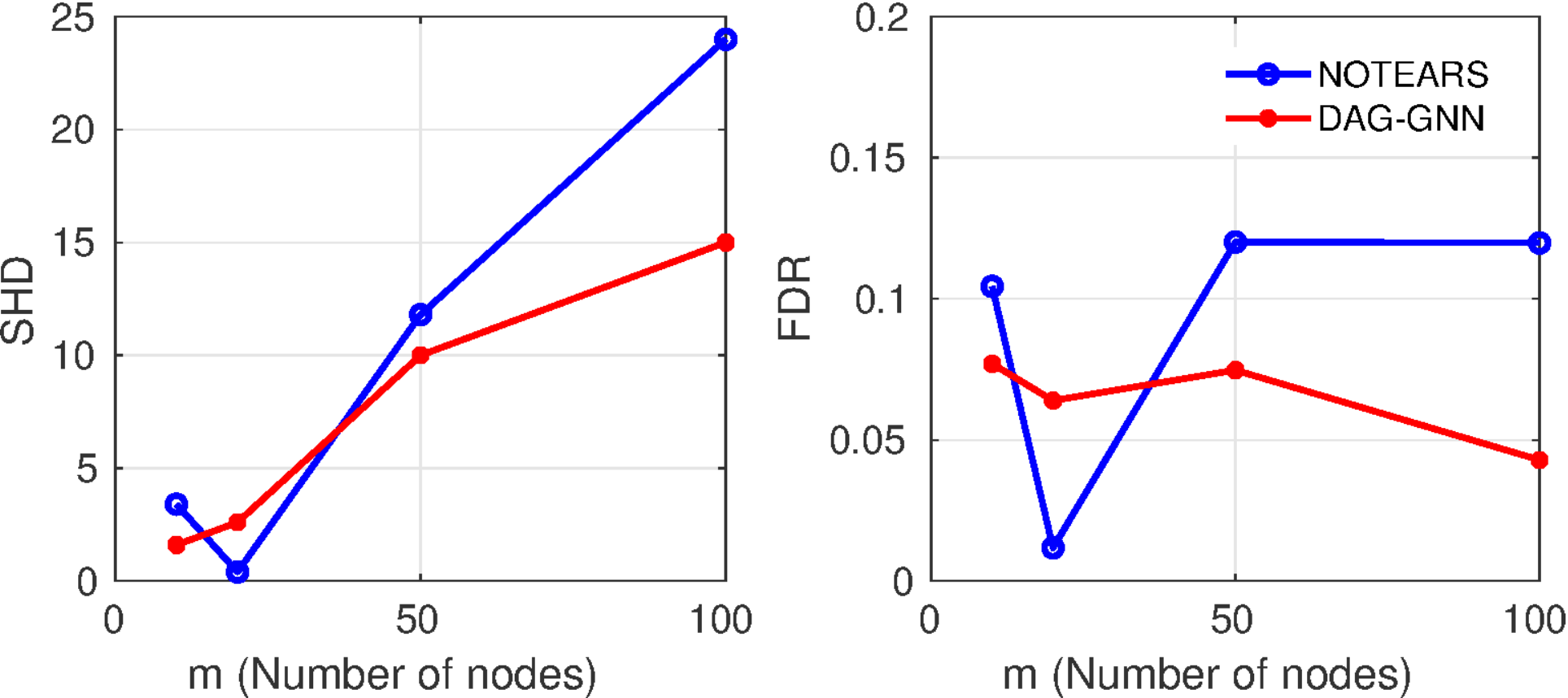}}}
\caption{Structure discovery in terms of SHD and FDR to the true graph, on synthetic data set generated by $x=A^Tx+z$.}
\label{fig:linear}
\end{figure}

\subsubsection{Nonlinear Case}\label{sec:nonlineardata}
We now consider data generated by the following model
\[
x = A^Th(x)+z,
\]
for some nonlinear function $h$. Taking first-order approximation $h(x)\approx h(0)\textbf{1}+h'(0)x$ (ignoring higher-order terms of $x$), one obtains an amendatory approximation of the graph adjacency matrix, $h'(0)A$. This approximate ground truth maintains the DAG structure, with only a scaling on the edge weights.

We take $h(x)=\cos(x+\mathbf{1})$ and plot the SHD and FDR in Figure~\ref{fig:cos}. one observes that DAG-GNN slightly improves over DAG-NOTEARS in terms of SHD. Further, FDR is substantially improved, by approximately a factor of three, which indicates that DAG-GNN tends to be more accurate on selecting correct edges. This observation is consistent with the parameter estimates shown in Figure~\ref{fig:cosG}, where the ground truth is set as $-\sin(1)A$. The heat map confirms that DAG-GNN results in fewer ``false alarms'' and recovers a relatively sparser matrix.

\begin{figure}[!ht]
\centering
\subfigure{\scalebox{0.38}[0.35]{\includegraphics{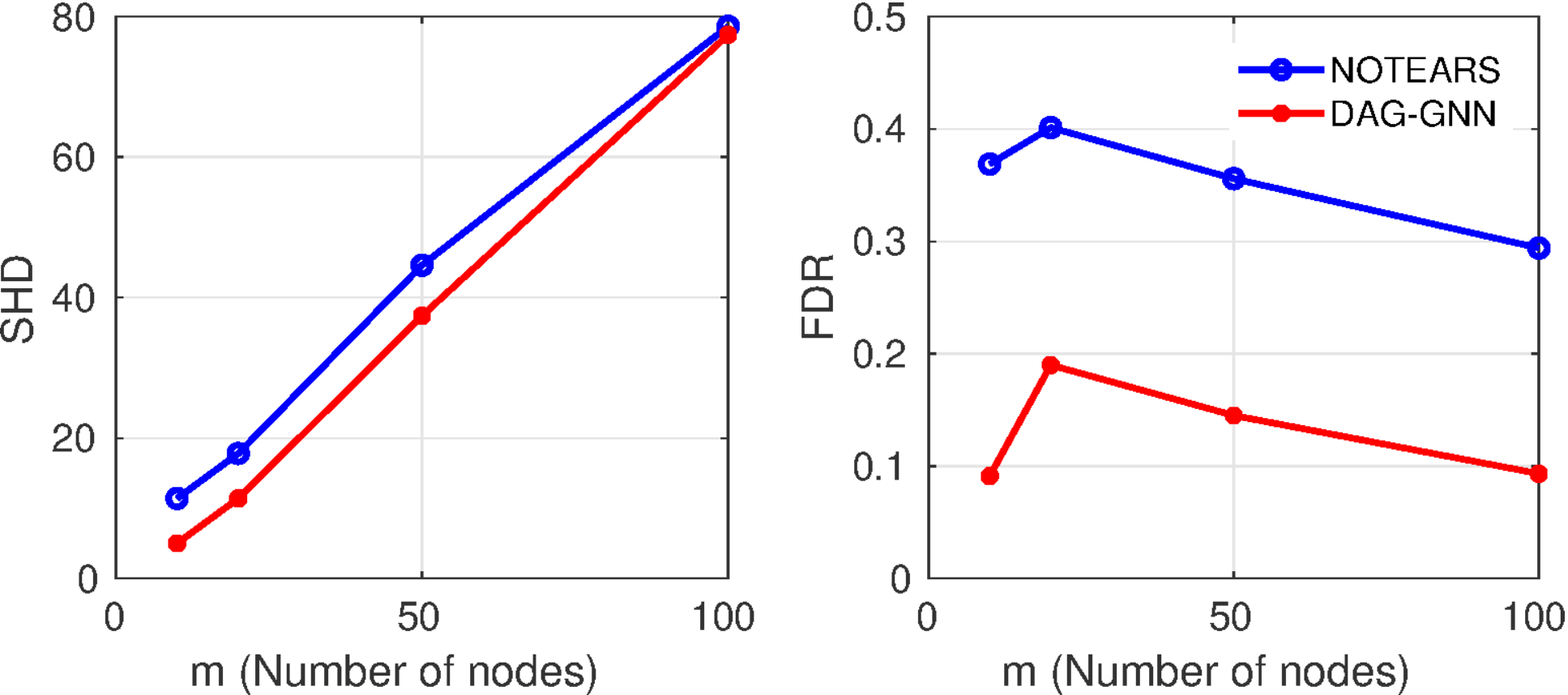}}}
\caption{Structure discovery in terms of SHD and FDR to the true graph, on synthetic data set generated by $x=A^T\cos(x+\mathbf{1})+z$.}
\label{fig:cos}
\end{figure}

\begin{figure}[!ht]
\centering
\subfigure{\scalebox{0.22}[0.22]{\includegraphics{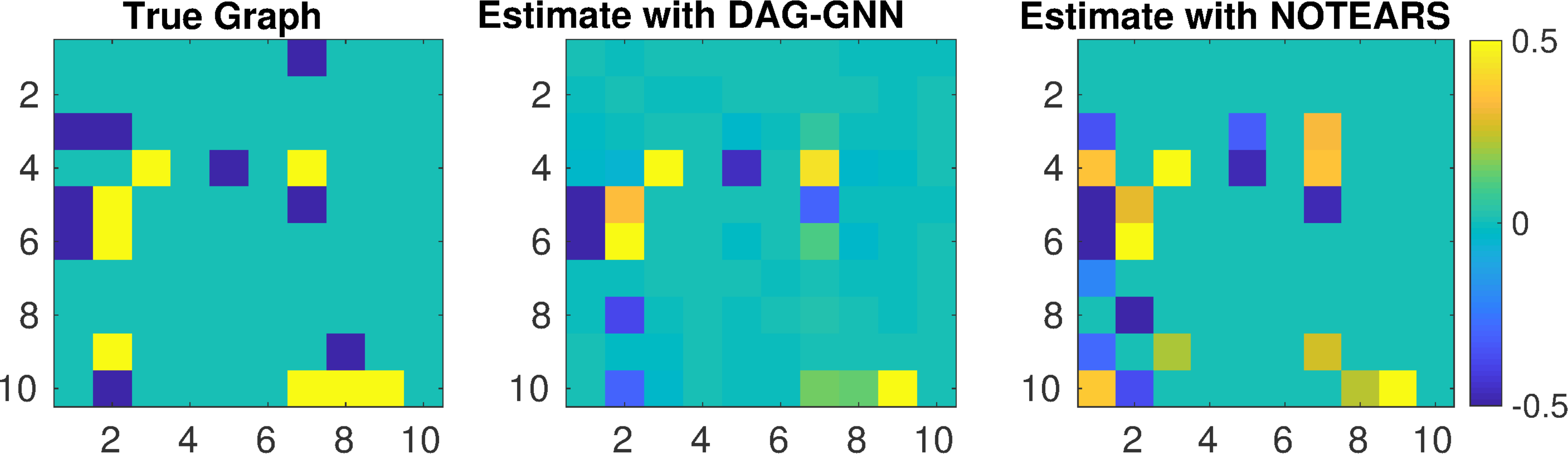}}}
\caption{Parameter estimates (before thresholding) of the graph on synthetic data set generated by $x=A^T\cos(x+\mathbf{1})+z$.}
\label{fig:cosG}
\end{figure}

We further experiment with a more complex nonlinear generation model, where the nonlinearity occurs after the linear combination of the variables, as opposed to the preceding case where nonlinearity is applied to the variables before linear combination. Specifically, we consider
\begin{displaymath}
x=2\sin(A^T(x+0.5\cdot\mathbf{1}))+A^T(x+0.5\cdot\mathbf{1})+z,
\end{displaymath}
and plot the results in Figure~\ref{fig:newsin}. One sees that with higher nonlinearity, the proposed method results in significantly better SHD and FDR than does DAG-NOTEARS.

\begin{figure}[!ht]
\centering
\subfigure{\scalebox{0.38}[0.35]{\includegraphics{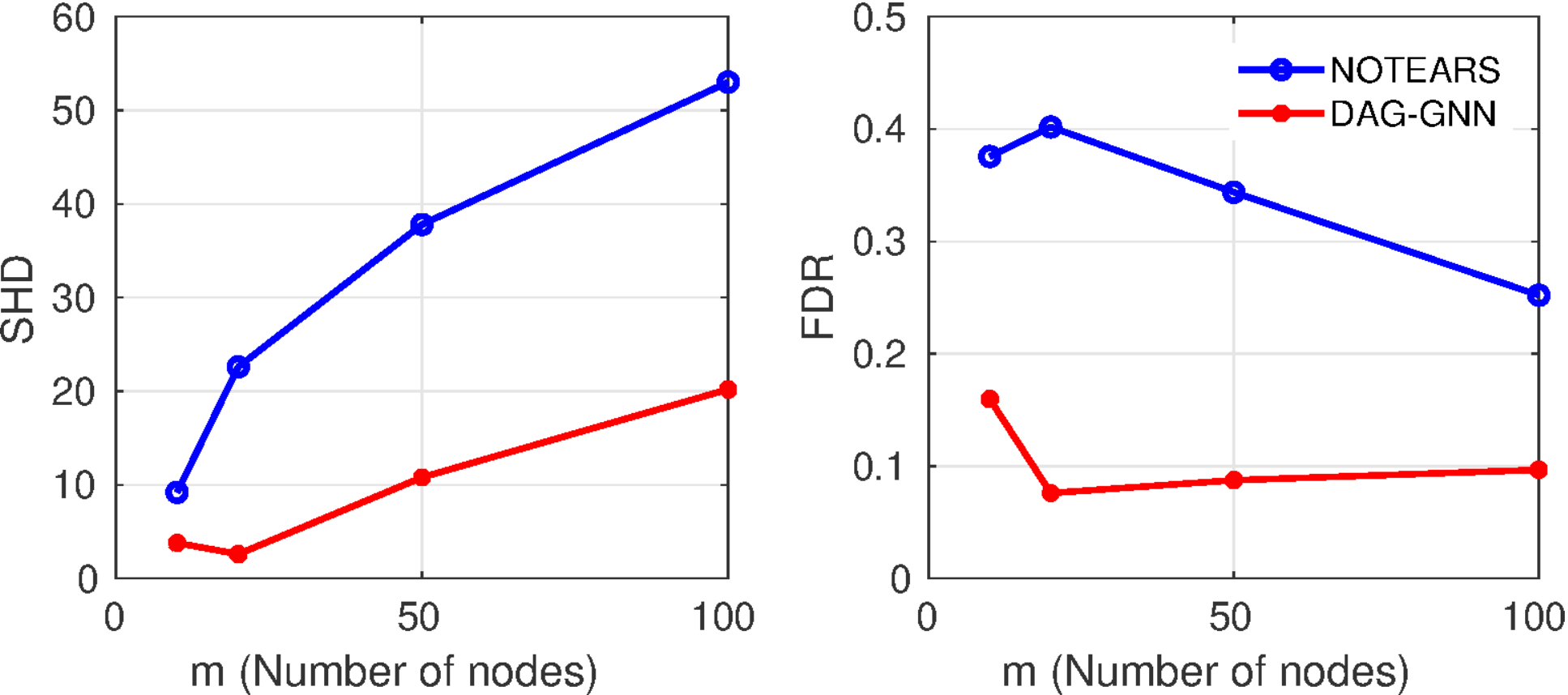}}}
\caption{Structure discovery in terms of SHD and FDR to the true graph, on synthetic data set generated by $x=2\sin(A^T(x+0.5\cdot\mathbf{1}))+A^T(x+0.5\cdot\mathbf{1})+z$.}
\label{fig:newsin}
\end{figure}

\subsubsection{Vector-Valued Case}\label{sec:multiddata}
The proposed method offers a modeling benefit that the variables can be vector-valued with $d>1$. Moreover, since $Z$ resides in the latent space of the autoencoder and is not interpreted as noise as in linear SEM, one may take a smaller column dimension $d_Z<d$ if he/she believes that the variables have a lower intrinsic dimension. To demonstrate this capability, we construct a data set where the different dimensions come from a randomly scaled and perturbed sample from linear SEM. Specifically, given a graph adjacency matrix $A$, we first construct a sample $\tilde{x}\in\real^{m\times1}$ from the linear SEM $\tilde{x}=A^T\tilde{x}+\tilde{z}$, and then generate for the $k$-th dimension $x^k=u^k\tilde{x}+v^k+z^k$, where $u^k$ and $v^k$ are random scalars from standard normal and $z^k$ is a standard normal vector. The eventual sample is $X=[x^1|x^2|\cdots|x^d]$.

We let $d=5$ and $d_Z=1$ and compare DAG-GNN with DAG-NOTEARS. The SHD and FDR are plotted in Figure~\ref{fig:multid}. The figure clearly shows the significantly better performance of the proposed method. Moreover, the parameter estimates are shown in Figure~\ref{fig:multidG}, compared against the ground-truth $A$. One sees that the estimated graph from DAG-GNN successfully captures all the ground truth edges and that the estimated weights are also similar. On the other hand, DAG-NOTEARS barely learns the graph.

\begin{figure}[!ht]
\centering
\subfigure{\scalebox{0.38}[0.35]{\includegraphics{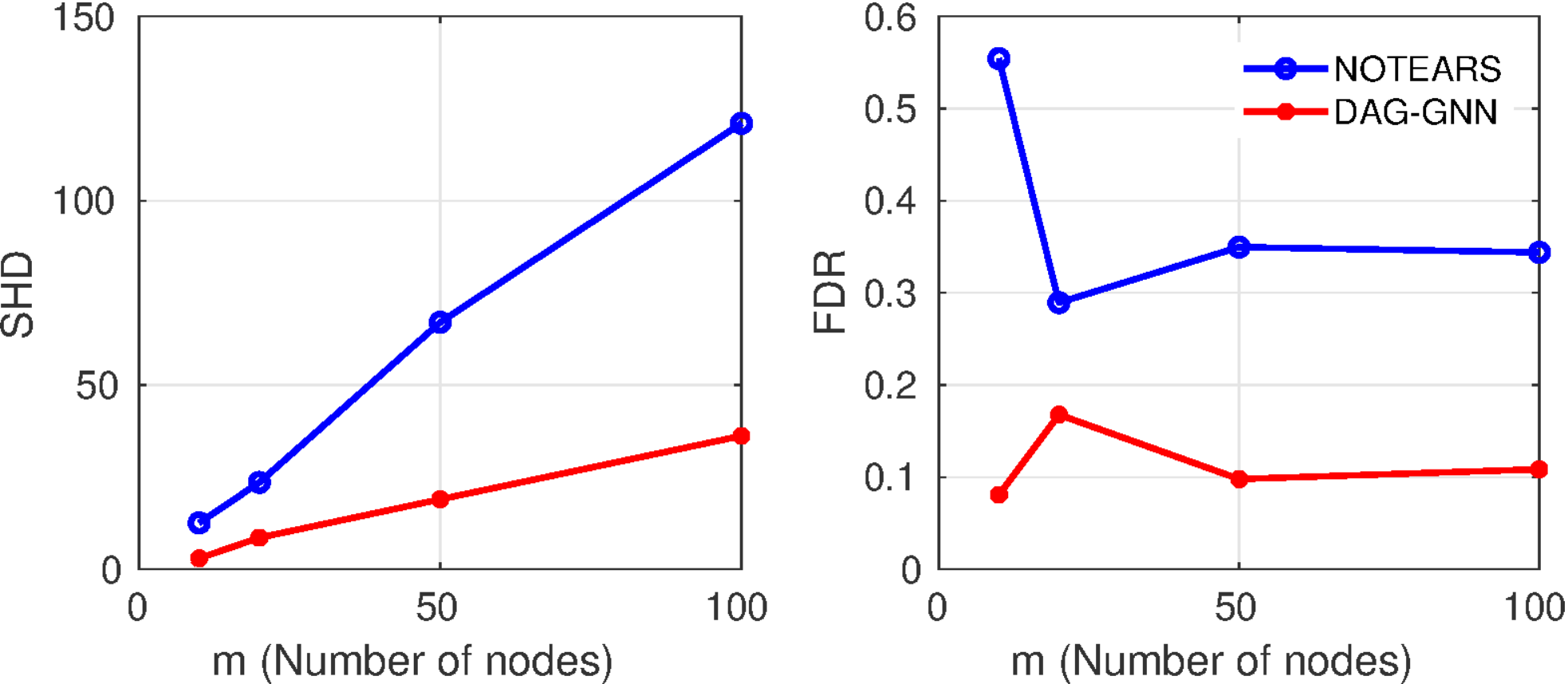}}}
\caption{Structure discovery in terms of SHD and FDR to the true graph, on synthetic vector-valued data set.}
\label{fig:multid}
\end{figure}

\begin{figure}[!ht]
\centering
\subfigure{\scalebox{0.22}[0.22]{\includegraphics{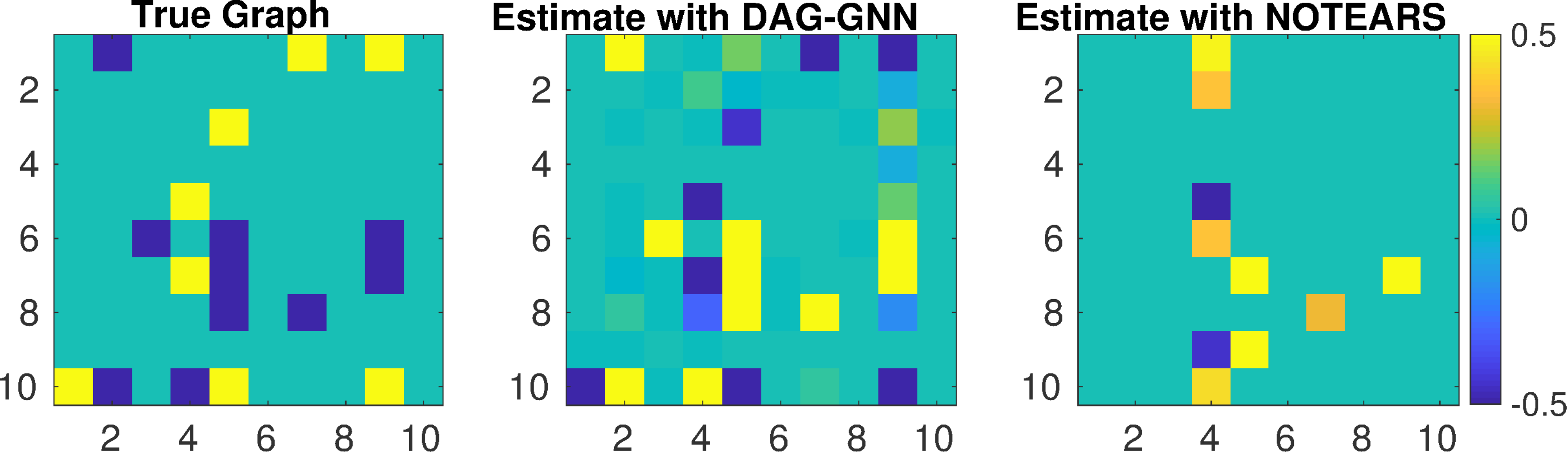}}}
\caption{Parameter estimates (before thresholding) of the graph on synthetic vector-valued data set.}
\label{fig:multidG}
\end{figure}

\vspace{-0.3cm}

\subsection{Benchmark Data Sets}\label{sec:data2}
A benefit of the proposed method is that it naturally handles discrete variables, a case precluded by linear SEM. We demonstrate the use of DAG-GNN on three discrete benchmark data sets: Child, Alarm, and Pigs~\citep{MMPCcor}. For comparison is the state-of-the-art exact DAG solver GOPNILP~\citep{cussens2016polyhedral}, which is based on a constrained integer programming formulation. We use 1000 samples for learning.

One sees from Table~\ref{Table:discrete} that our results are reasonably close to the ground truth, whereas not surprisingly the results of GOPNILP are nearly optimal. The BIC score gap exhibits by DAG-GNN may be caused by the relatively simple autoencoder architecture, which is less successful in approximating multinomial distributions. Nevertheless, it is encouraging that the proposed method as a unified framework can handle discrete variables with only slight changes in the network architecture.

\begin{table}[!ht]
\centering
\caption{BIC scores on benchmark datasets of discrete variables.}
\label{Table:discrete}
\vskip 0.1in
\small
\begin{tabular}{ccccc}
\hline
Dataset& $m$ &Groundtruth&GOPNILP&DAG-GNN\\
\hline
Child & 20 & \texttt{-1.27e+4} & \texttt{-1.27e+4} & \texttt{-1.38e+4}\\
Alarm & 37 & \texttt{-1.07e+4} & \texttt{-1.12e+4} & \texttt{-1.28e+4}\\
Pigs & 441 & \texttt{-3.48e+5} & \texttt{-3.50e+5} & \texttt{-3.69e+5}\\
\hline
\end{tabular}
\end{table}

\subsection{Applications}\label{sec:data3}
We consider a bioinformatics data set~\citep{sachs2005causal} for the discovery of a protein signaling network based on expression levels of proteins and phospholipids. This is a widely used data set for research on graphical models, with experimental annotations accepted by the biological research community. The data set offers continuous measurements of expression levels of multiple phosphorylated proteins and phospholipid components in human immune system cells, and the modeled network provides the ordering of the connections between pathway components. Based on $n=7466$ samples of $m=11$ cell types, \citet{sachs2005causal} estimate 20 edges in the graph.

In Table~\ref{Table:protein}, we compare DAG-GNN with DAG-NOTEARS as well as FSG, the fast greedy search method proposed by~\citet{ramsey2017million}, against the ground truth offered by~\citet{sachs2005causal}. The proposed method achieves the lowest SHD. We further show in Figure~\ref{fig:protein} our estimated graph. One observes that it is acyclic. Our method successfully learns 8 out of 20 ground-truth edges (as marked by red arrows), and predicts 5 indirectly connected edges (blue dashed arrows) as well as 3 reverse edges (yellow arrows).

\begin{table}[!ht]
\centering
\caption{Results on protein signaling network: comparison of the predicted graphs with respect to the ground truth.} \label{Table:protein}
\vskip 0.1in
\begin{tabular}{ccc}
\hline
Method& SHD & \# Predicted edges\\
\hline
FGS&22&17\\
NOTEARS&22&16\\
DAG-GNN&19&18\\
\hline
\end{tabular}
\end{table}

\begin{figure}[!ht]
\centering
\subfigure{\scalebox{0.6}[0.55]{\includegraphics{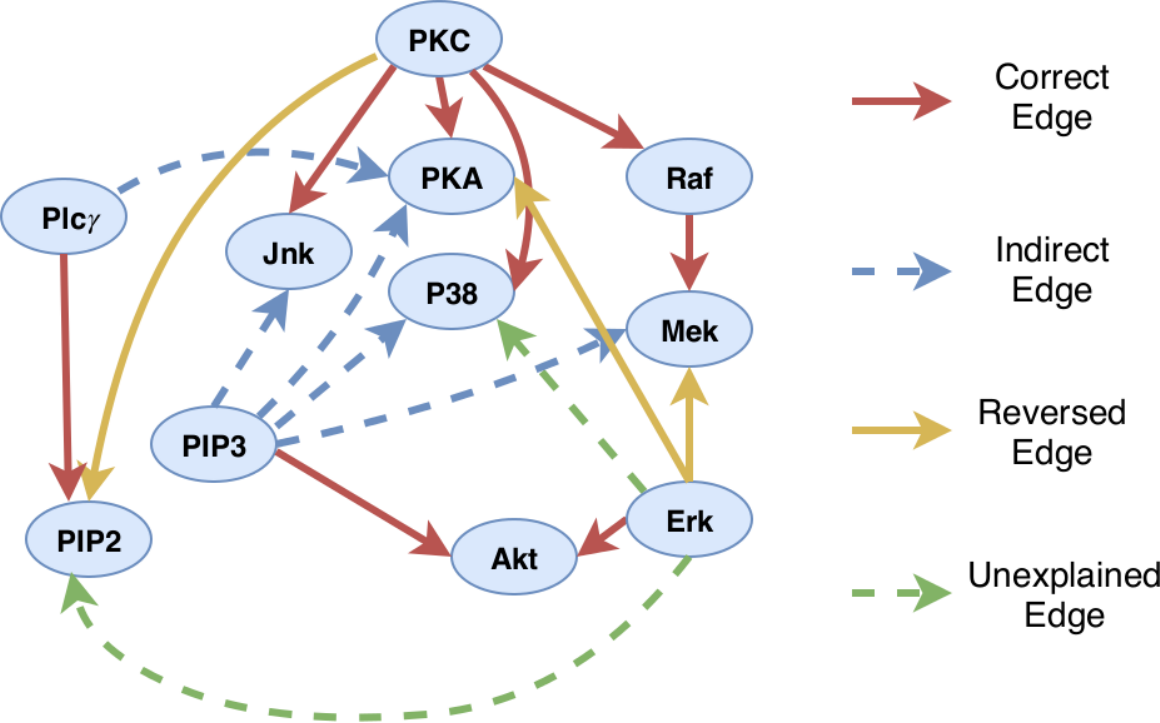}}}
\caption{Estimate protein signaling network.}
\label{fig:protein}
\end{figure}

For another application, we develop a new causal inference task over relations defined in a knowledge base (KB) schema. The task aims at learning a BN, the nodes of which are relations and the edges indicate whether one relation suggests another. For example, the relation person/Nationality may imply person/Language, because the spoken language of a person naturally associates with his/her nationality. This task has a practical value, because most existing KBs are constructed by hand. The success of this task helps suggest meaningful relations for new entities and reduce human efforts. We construct a data set from FB15K-237~\citep{toutanova2015representing} and list in Table~\ref{tab:kb} a few extracted causal relations. Because of space limitation, we defer the details and more results in the supplementary material. One sees that these results are quite intuitive. We plan a comprehensive study with field experts to systematically evaluate the extraction results.

\begin{table}[!ht]
\centering
\caption{Examples of extracted edges with high confidence.
}
\label{tab:kb}
\vskip 0.1in
\small
\begin{tabular}{rcl}
\toprule[0.8pt]
film/ProducedBy &$\Rightarrow$ &film/Country \\
film/ProductionCompanies & $\Rightarrow$ &film/Country \\
\midrule
person/Nationality& $\Rightarrow$& person/Languages \\
person/PlaceOfBirth& $\Rightarrow$& person/Languages \\
\midrule
person/PlaceOfBirth& $\Rightarrow$& person/Nationality \\
person/PlaceLivedLocation& $\Rightarrow$& person/Nationality\\
\toprule[0.8pt]
\end{tabular}
\vspace{-0.3cm}
\end{table}

\section{Conclusion}
DAG structure learning is a challenging problem that has long been pursued in the literature of graphical models. The difficulty, in a large part, is owing to the NP-hardness incurred in the combinatorial formulation. \citet{Zheng2018} propose an equivalent continuous constraint that opens the opportunity of using well developed continuous optimization techniques for solving the problem. In this context, we explore the power of neural networks as functional approximators and develop a deep generative model to capture the complex data distribution, aiming at better recovering the underlying DAG with a different design of the objective function. In particular, we employ the machinery of variational autoencoders and parameterize them with new graph neural network architectures. The proposed method handles not only data generated by parametric models beyond linear, but also variables in general forms, including scalar/vector values and continuous/discrete types. We have performed extensive experiments on synthetic, benchmark, and application data and demonstrated the practical competitiveness of the proposal.


\bibliography{bib,jie}
\bibliographystyle{icml2019}

\newpage
\appendix

\section{Proofs}

\begin{proof}[Proof of Theorem~\ref{thm:tr}]
  Let $B=A\circ A$. Clearly, $B$ is nonnegative. The binomial expansion reads
  \[
  (I+\alpha B)^m = I + \sum_{k=1}^m \binom{m}{k}\alpha^kB^k.
  \]
  It is known that there is a cycle of length $k$ if and only if $\tr(B^k)>0$ when $B\ge0$. Because if there is a cycle then there is a cycle of length at most $m$, we conclude that there is no cycle if and only if $\tr[(I+\alpha B)^m]=\tr(I)=m$.
\end{proof}

\begin{proof}[Proof of Theorem~\ref{thm:exp}]
  Write
  \[
  (1+\alpha|\lambda|)^m=\left(1+\frac{c|\lambda|}{m}\right)^m.
  \]
  For given $c$ and $|\lambda|$, the right-hand side of the equality is a function of $m$. This function monotonically increases for positive $m$ and has a limit $e^{c|\lambda|}$. Hence, for any finite $m>0$, $(1+\alpha|\lambda|)^m\le e^{c|\lambda|}$.
\end{proof}

\section{Structure Learning over KB Relations}

We construct the data set from triples in
FB15K-237 \cite{toutanova2015representing}, which is a subset of FreeBase with approximately 15k entities and 237 relations. Each sample corresponds to an entity and each variable corresponds to a relation in this knowledge base. Each sample has on average 7.36 relations (i.e. 7.36 non-zero entries in each row).

Table \ref{tab:kb_results} gives additional examples learned by our model with highest confidence scores. 
For each target relation on the right-hand side, we show the highest ranked relations within the same domain (i.e. the contents in the field before ``/'' such as ``film'' and ``tvProgram'').
On the left-hand side, we omit the relations that are common to the associated entity types, e.g. ``profession'' and ``gender'' to persons and ``genre'' to films, because almost all entities with these types will contain such a relation.

\begin{table*}[t!]
\centering
\caption{(Continued from Table~\ref{tab:kb}) Examples of extracted edges with high confidence. The dot $\cdot$ appearing in $R_1.R_2$ means that the sample entity is connected to a virtual node (i.e. compound value types in FreeBase) via relation $R_1$, followed by a relation $R_2$ to a real entity.}
\label{tab:kb_results}
\vskip 0.1in
\begin{tabular}{rcl}
\toprule[0.8pt]
film/ProducedBy &$\Rightarrow$ &film/Country \\
film/ProductionCompanies & $\Rightarrow$ &film/Country \\
\midrule
tvProgram/CountryOfOriginal&$\Rightarrow$& tvProgram/Language \\
tvProgram/RegularCast.regularTv/AppearanceActor&$\Rightarrow$&tvProgram/Language\\
\midrule
person/Nationality& $\Rightarrow$& person/Languages \\
person/PlaceOfBirth& $\Rightarrow$& person/Languages \\
\midrule
person/PlaceOfBirth& $\Rightarrow$& person/Nationality \\
person/PlaceLivedLocation& $\Rightarrow$& person/Nationality\\
\midrule
organization/Headquarters.mailingAddress/Citytown& $\Rightarrow$&organization/PlaceFounded\\
organization/Headquarters.mailingAddress/StateProvince& $\Rightarrow$&organization/PlaceFounded\\
\toprule[0.8pt]
\end{tabular}
\end{table*}

\end{document}